\documentclass[12p]{article}

\usepackage[letterpaper,
            left=1.7in,
            right=1.7in,
            top=1.2in,
            bottom=1.2in]{geometry}

\usepackage{caption}
\usepackage{placeins}

\title{
Manifold-Guided Motion Planning 
\\for Tight Assemblies
}
\author{Dror Livnat$^{\dag}$ \and Michael M. Bilevich$^{\dag}$ \and Michal Kleinbort$^{\dag}$ \and Dan Halperin$^{\dag}$}
\date{}

\usepackage{amssymb}
\usepackage{amsmath}
\usepackage{amsfonts}

\usepackage{amsthm}
\usepackage{listings}
\usepackage{hyperref}
\usepackage{graphicx}
\usepackage{wrapfig}
\usepackage{xcolor}
\usepackage{mathtools}
\usepackage{algorithm}
\usepackage{algpseudocode}
\usepackage{textcomp}
\usepackage{cite}
\usepackage{booktabs}

\usepackage{tikz}
\usetikzlibrary{shapes.geometric, arrows.meta, positioning, fit, backgrounds}

\newtheorem{theorem}{Theorem}[section]

\newtheorem{lemma}[theorem]{Lemma}
\usepackage{graphicx}
\usepackage{algorithm}
\usepackage{xcolor}
\usepackage{comment}
\usepackage{xspace}
\usepackage{mathtools}
\usepackage{hyperref}
\usepackage{threeparttable}

\usepackage{booktabs}
\usepackage{siunitx}
\usepackage{subfig}
\usepackage{diagbox}
\newcommand{\cmg}{\texttt{CMG-RRT}\xspace}
\newcommand{\rrt}{\texttt{RRT}\xspace}
\newcommand{\trrrt}{\texttt{TR-RRT}\xspace}
\newcommand{\prm}{\texttt{PRM}\xspace}

\newcommand{\cspace}{\ensuremath \mathcal{C}}
\newcommand{\free}{\ensuremath \mathcal{F}_{\delta}}
\newcommand{\B}{\ensuremath \mathcal{B}^{\cspace}}

\newcommand{\SE}{\mathrm{SE}}
\newcommand{\SO}{\mathrm{SO}}

\newcommand{\bbP}{\mathbb{P}}
\newcommand{\bbR}{\mathbb{R}}
\newcommand{\bbS}{\mathbb{S}}

\newcommand{\calU}{\mathcal{U}}

\newcommand{\qstart}{q_\mathrm{start}}
\newcommand{\qgoal}{q_\mathrm{goal}}
\newcommand{\qnear}{q_\mathrm{near}}
\newcommand{\qnew}{q_\mathrm{new}}
\newcommand{\fsample}{\textsc{Sample}}

\newcommand{\fextend}{\textsc{Extend}}

\newcommand{\frefine}{\textsc{Refine}}
\newcommand{\fupdate}{\textsc{Update}}

\newcommand{\fdirection}{\textsc{Direction}}
\newcommand{\frandomrotate}{\textsc{RandomRotate}}

\newcommand{\Texpand}{T_{\mathrm{expand}}}
\newcommand{\Trefine}{T_{\mathrm{refine}}}
\newcommand{\Tsearch}{T_{\mathrm{search}}}
\newcommand{\factorref}{\mathrm{factor}_{\mathrm{refine}}}
\newcommand{\LIMIT}{\mathrm{LIMIT}}

\newcommand{\distc}{d_{\scriptscriptstyle{\cspace}}}

\newcommand{\drthree}{d_{\scriptscriptstyle{\bbR^3}}}
\newcommand{\dsone}{d_{\scriptscriptstyle{\bbS^1}}}
\newcommand{\dsonei}{d_{\scriptscriptstyle{\bbS^1[i]}}}

\newcommand{\robotsdf}{F_{M_2\to M_1}}

\newcommand{\Nrot}{N_\mathrm{rot}}

\begin{document}

\maketitle
\begingroup
\renewcommand\thefootnote{\fnsymbol{footnote}}
\footnotetext[2]{Blavatnik School of Computer Science and Artificial Intelligence, Tel-Aviv University, Israel. This work has been supported in part by the Israel Science
Foundation (grant no~3598/25), 
by the Blavatnik Computer Science Research Fund, 
and by the Shlomo Shmelzer Institute for 
Smart Transportation at Tel Aviv University.}
\endgroup

\begin{abstract}
    Motion planning for rigid-body assembly poses a fundamental challenge in robotics due to tight geometric constraints. In such scenarios, feasible motions often require passing through (near-)zero
    clearance
    configurations in which the parts are tightly constrained by contact. In this work, we introduce Critical-Manifold Guided RRT (CMG-RRT), a sampling-based planner designed specifically for tight assembly problems.
    Our key observation is that in tight assemblies, valid solution paths lie on or near a critical manifold: the subset of configuration space consisting of poses with at least one contact point between parts.
    CMG-RRT guides exploration by adaptively biasing sampling toward neighborhoods of the critical manifold using a hierarchical subdivision of the configuration space. We prove that CMG-RRT is probabilistically complete under standard clearance assumptions. Empirical evaluation on challenging rotational assembly benchmarks demonstrates a $100\%$ success rate across all tested instances, including, to the best of our knowledge, the first fully automatic solution of the \emph{Elk} disentanglement puzzle. Our open source software is available through our project page: \url{https://www.cgl.cs.tau.ac.il/projects/tight-assembly-planning}
\end{abstract}

\begin{figure}[ht]
    \centering
    \includegraphics[width=1.0\linewidth]{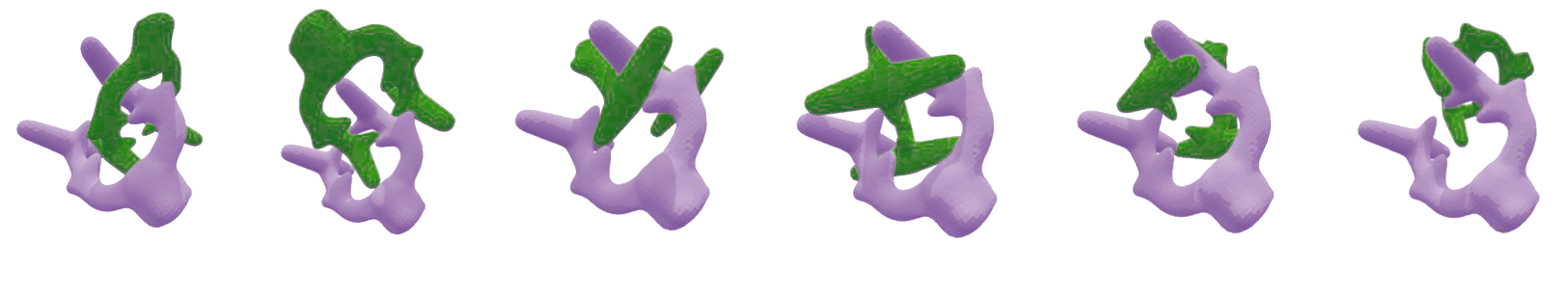}
    \caption{Snapshots of the \emph{Elk} disentanglement puzzle, previously unsolved (to the best of our knowledge), showing a complete solution generated by our algorithm. From left to right: the start configuration, traversal through four narrow C-space tunnels interleaved with wide free-space regions, and the goal configuration.
    }
    \label{fig:elk-time-series}
\end{figure}

\section{Introduction}
\label{sec:introduction}
Sampling-based motion planning is a fundamental paradigm in robotics, used in applications such as autonomous navigation, manipulation, assembly, and numerous others~\cite{lavalle2006planning_book}. 
Methods such as Probabilistic Roadmaps (PRM) and Rapidly-Exploring Random Trees (RRT) have become standard due to their scalability to high-dimensional configuration spaces and their probabilistic completeness guarantees~\cite{kavraki1996probabilistic_prm,lavalle1998rapidly,lavalle2001rapidly,xu2024recent}.
When feasible motions have sufficient clearance from obstacles,
uniform random sampling in configuration space enables these planners to efficiently discover collision-free paths.

A longstanding difficulty arises in \emph{tight} motion-planning problems, where any valid solution must pass through regions of extremely small measure in configuration space, commonly referred to as \emph{narrow passages} or \emph{tunnels}~\cite{hsu1999path,kavraki1998analysis,DBLP:journals/tase/SalzmanHH15}.
Such situations are common in robotic assembly, particularly for rigid parts that require tightly coupled translation and rotation.
In these settings, uniform sampling becomes highly inefficient: the overwhelming majority of samples lie in regions irrelevant to the solution, while the critical regions that enable progress are sampled with vanishing probability.
This failure mode is a canonical limitation of sampling-based planners and is often illustrated through ``bug-trap'' examples~\cite{lavalle2006planning_book}.

To address narrow passages, a variety of nonuniform sampling strategies have been proposed, including bridge tests, obstacle-based sampling, spatial subdivisions,and Gaussian or medial-axis inspired heuristics~\cite{hsu2003bridge_test,sun2005narrow_passage_sampling,yan2013path,zhang2007hybrid}.
While these methods can improve performance on certain instances, they typically rely on geometric heuristics that are unreliable in high-dimensional spaces and offer no general guarantees for complex, contact-rich motions. 
The Soft Subdivision Search (SSS) framework for motion planning~\cite{yap2013soft,yap2015soft,zhang2024theory} has guarantees, but applying it to sufficiently complex assemblies or robots seems (as of now) a prohibitively hard task.

Many tight assemblies may require sliding or rolling along obstacle boundaries, rather than maintaining clearance throughout the motion.
This observation has motivated work on compliant and contact-aware motion planning~\cite{friedman1996compliant,lefebvre2005active_compliant_motion,de1988compliant}, as well as physics- and simulation-based planners that can exploit contact and near-contact interactions to solve tight disentanglement tasks~\cite{tian2022assemble_them_all}. Feature-driven tunnel discovery methods similarly leverage local geometric structure to guide search toward narrow passages without assuming continuous contact~\cite{zhang2020cspace_tunnel_discovery}. 

While these approaches can be effective for problems dominated by tight interactions or exhibiting specific structural cues, they are less suited to mixed settings in which wide free-space motion is interleaved with short but critical tight transitions whose locations are not known in advance.

Recent work has explored learning-based and data-driven techniques to guide sampling toward narrow passages or critical connectivity regions~\cite{lee2022adaptive_experience_sampling,li2023sample_driven_connectivity_learning} 
Another influential line of work studies motion planning on \emph{constraint manifolds}.
Projection-based planners such as CBiRRT~\cite{berenson2009cbirrt_manipulation} and atlas-based methods, including AtlasRRT and its variants~\cite{jaillet2017atlas_rrt,voss2017atlas_plus_x}, explore lower-dimensional manifolds defined implicitly by kinematic or task constraints.
Comprehensive surveys~\cite{kingston2019exploring_implicit_spaces,kingston2018sampling_constrained_survey} have formalized this perspective and shown its effectiveness for problems with fixed, known constraints such as closed kinematic chains.
In contrast, contact constraints in assembly planning are neither fixed nor known a priori: the set of active contacts changes along the motion, and the corresponding manifolds appear and disappear as the robot moves.

In prior work, we introduced TouchRoll-RRT (\trrrt)~\cite{livnat2024tight},
a sampling-based planner designed for mixed wide-tight settings.
\trrrt combines standard RRT exploration in free space with a contact-aware extension procedure guided by a signed distance function (SDF).
When the search reaches the vicinity of obstacles, \trrrt maintains multiple contact points and advances by projecting steering directions onto the tangent space of a local \emph{contact critical manifold}, followed by a retraction step.
This reduces the effective degrees
of freedom from six to five, four, and at times even three or two, and enables traversal of narrow passages that outperforms 
standard RRT.

\begin{figure}[h]
    \centering   \includegraphics[width=1.0\linewidth]{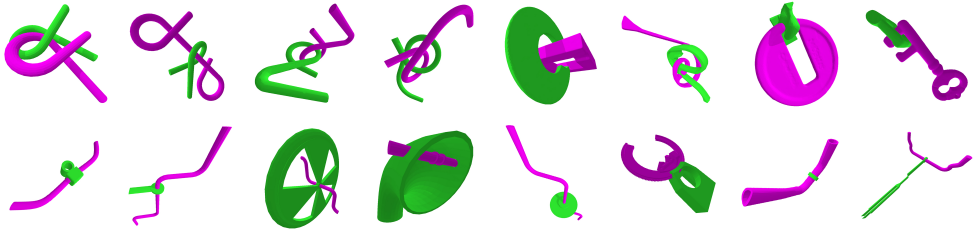}
    \caption{Dataset provided by Tian~et~al.~\cite{tian2022assemble_them_all}. Excerpts from the more challenging \emph{Rotational} benchmarks:  \texttt{Puzzles} in the first row and \texttt{Others} in the second row.}
    \label{fig:dataset}
\end{figure}

However, \trrrt retains a fundamental mismatch between sampling and exploration.
While its extension step explicitly follows the critical manifold, its sampling step remains uniform in the six-dimensional space $\SE(3)$. 
As a result, only samples that happen to induce extensions near the manifold benefit from the contact-aware machinery.
Many samples are generated in regions that do not contribute to progress towards the goal 
causing the tree to repeatedly grow into misleading dead ends---an effect similar to bug-trap behavior~\cite{lavalle2006planning_book}.
This sampling bias limits performance and prevents \trrrt from solving particularly challenging instances.

In this work, we present \emph{Critical-Manifold Guided RRT} (\cmg), which resolves this limitation by explicitly biasing sampling toward the critical manifolds. We consider the configuration space $\cspace = \SE(3)$, the six-dimensional manifold of rigid body transformations in $\bbR^3$.
We maintain an adaptive subdivision of the configuration space into axis-aligned boxes and discard boxes that are provably far from the contact manifold,
using an SDF-based distance oracle.
Sampling is restricted to the remaining boxes, whose union is contained in an increasingly small $D$-neighborhood of the critical manifold as the box diagonal $D$ decreases. 
In addition, we occasionally apply random rotations to the entire system, an idea already in use for other purposes in the field, such as efficiently finding pairs of nearest neighbors~\cite{aiger2014random,kleinbort2015efficient}. In this work it helps avoiding gimbal locks~\cite{hemingway2018perspectives,lavalle2006planning_book}, as well as enables the completeness proof. 

Search and refinement are interleaved adaptively, so that easy problems incur little overhead, while harder problems trigger additional refinement that sharpens the sampling distribution.

As we prove below, this novel manifold-guided sampling scheme allows \cmg to be \emph{probabilistically complete}; that is, under standard assumptions of positive clearance, the probability that \cmg finds a solution approaches one as the number of iterations grows.
Moreover, \cmg accelerates planning on problems already solvable by \trrrt and, crucially, solves significantly harder instances.
A notable example is the \emph{Elk} disentanglement puzzle\footnote{A commercially available cast-metal disentanglement puzzle designed by Nobuyuki Yoshigahara and marketed by Hanayama of Japan.}, previously identified as particularly challenging due to the absence of exploitable geometric features and the presence of multiple dead-end tunnels~\cite{livnat2024tight,zhang2020cspace_tunnel_discovery}.
Using \cmg, we demonstrate---%
to the best of our knowledge---the first fully automatic solution of this puzzle using a general-purpose motion planner without handcrafted features.

\smallskip 

\noindent
\textbf{Our Contribution}
Our contributions are as follows: (i)~We introduce \cmg, a sampling-based planner that adaptively concentrates samples near contact critical-manifolds. (ii)~We extend previous probabilistic completeness proofs to an $\SE(3)$ parametrization with non-Euclidean topology, and show how choices in our algorithm directly benefit the completeness proof. (iii)~We present experimental results, demonstrating the effectiveness of our method over prior methods. (iv)~Our software is open-source and publicly available online.~\footnote{Project page: \url{www.cgl.cs.tau.ac.il/projects/tight-assembly-planning}}

In particular, we show that our method solves challenging assembly problems, including the Elk puzzle. In~\cite{zhang2020cspace_tunnel_discovery}, the authors mention the Elk puzzle as a tough challenge that their algorithm fails to solve. In~\cite{livnat2024tight}, we tried applying TR-RRT to the Elk puzzle, and although it was able to go through individual c-space tunnels, it failed to solve the entire puzzle due to the c-space uniform sampling. This is the first time that an algorithm solving the Elk puzzle is presented, as far as we know.

\smallskip 

\noindent
\textbf{Organization}
The remainder of the paper is organized as follows.
Section~\ref{sec:preliminaries} reviews the background and formalizes the problem.
Section~\ref{sec:CMG-RRT} presents the \cmg algorithm.
Section~\ref{sec:completeness} provides a probabilistic completeness proof.
Experimental results are reported in Section~\ref{sec:results}, followed by discussion and future directions in Section~\ref{sec:discussion}.

\section{Preliminaries and Problem Statement}
\label{sec:preliminaries}

\subsection{The \texorpdfstring{$\SO(3)$}{SO(3)} and \texorpdfstring{$\SE(3)$}{SE(3)} Groups and the Configuration Space}
\label{ssec:so3}

When dealing with rigid body transformations, there are two groups that naturally arise~\cite{hashim2019special,selig2013geometrical}---the special orthogonal group $\SO(3)$ and the special Euclidean group $\SE(3)$. Both are Lie groups and are extensively researched in mathematics, physics, and robotics. 

As previously mentioned, in this work, the configuration space~\cite{lavalle2006planning_book} 
is $\cspace = \SE(3)$, as we solve for the assembly of a part among other static
parts.

We note that $\SO(3)$ is a three-dimensional manifold. However, there are many common methods of representing orientations, and each has its advantages and shortcomings. Examples are unit quaternions, angle axis, and Euler angles, to name a few~\cite{diebel2006representing}. In this work, we use only the Euler angle representation, and in particular, the \emph{rpy} (roll-pitch-yaw) representation.

Recall the notion of the Haar measure on a group~\cite{ecker2024haar}, which is a generalization of the Lebesgue measure to locally compact groups.
The Haar measure is a measure that is invariant under group multiplication and is unique up to a constant scaling factor. For example, the Lebesgue measure on $\bbR^n$ is also a Haar measure when $\bbR^n$ is viewed as an additive group. Since $\SO(3)$ is a \emph{compact} Lie group, its Haar measure can be normalized to a probability measure, providing a notion of uniform sampling over rotations.

We focus on the \emph{rpy} representation, let $\phi,\psi\in [-\pi,\pi)$ be the roll and yaw, respectively, and $\theta\in [-\pi/2,\pi/2]$ be the pitch. Then the volume element of the Haar measure using this representation~\cite{ecker2024haar} 
is
\begin{align}
    \frac{1}{8\pi^2}\cos\theta\ d\phi \ d\theta \ d\psi\;.
\end{align}

Finally, we note that due to this parametrization of $\SO(3)$, we regard our configuration space as $\cspace = \SE(3) \simeq \bbR^3 \times [-\pi,\pi)\times [-\pi/2,\pi/2]\times[-\pi, \pi) \subset \bbR^6$. Hence, in this work, 
we regard points in the Lie group $\SE(3)$ as their six-dimensional (translation + \emph{rpy}) representation.

\subsection{Pseudo-metric Spaces}
\label{sub:pseudometric}

Metric spaces~\cite{linial2003finite,searcoid2007metric} 
arise in many different applications. A metric space $(X,d)$ is some set $X$ equipped with a function $d:X\times X \to \bbR$ called \emph{metric} that satisfies three axioms: (i) (\emph{symmetry}) $d(x,y)=d(y,x)$, (ii) (\emph{positivity}) $d(x,y)\geq 0$ and equality if and only if $x=y$, and (iii) for any $z\in X$, (\emph{triangle inequality}) $d(x,y) \leq d(x,z)+d(z,y)$.

Similarly, we can define a \emph{pseudo-metric} space~\cite{howes1995modern}, by replacing axiom (ii) with a softer requirement: for any $x,y\in X$, $d(x,y)\geq 0$ and $d(x,x) = 0$. Hence, we possibly allow for different points $x\neq y$ to have $d(x,y)=0$.

One example is $\drthree : \bbR^6 \times \bbR^6 \to \bbR$, defined by $\drthree(x,y) = \sqrt{\sum_{i=1}^3(x_i-y_i)^2}$.
Clearly, it is symmetric and has the triangle inequality (since the Euclidean norm on $\bbR^3$ has the triangle inequality). However, if $e_1,\dots,e_6$ are the standard basis of $\bbR^6$, then $\drthree(e_4,e_5)=0$, meaning the $\drthree$ is not a metric, but is a pseudo-metric.

Let us define the \emph{great-circle metric} on $\bbS^1$, which we parametrize as $\bbS^1 \simeq [-\pi, \pi)$. Take $\dsone : \bbS^1\times\bbS^1\to \bbR$, 
\begin{align}
    \dsone(x,y)=\min (|x-y|, 2\pi-|x-y|)\;.
\end{align}

\begin{lemma}
    The function $\dsone$ is a metric on $\bbS^1\simeq [-\pi,\pi)$.
\end{lemma}

\begin{proof}
    Notice that this function is the distance of the shorter arc length (in radians) between $x$ and $y$.
\end{proof}

We can extend the metric $\dsone$ to a pseudo-metric on $\cspace$: choose some index $i\in\{1,\dots,6\}$. Then $\dsonei:\cspace\times\cspace \to \bbR$, defined as $\dsonei(x,y) = \dsone(x_i,y_i)$ is indeed a pseudo-metric on $\cspace$.

Notice the following lemma on pseudo-metrics:

\begin{lemma}
\label{lem:combine-pseudometrics}
    Let $d_1,d_2:X\times X\to\bbR$ be pseudo-metrics. Then $d_3:X\times X\to\bbR$, defined by 
    \begin{align}
        d_3(x,y)\coloneqq 
        \sqrt{
            \left(d_1(x,y)\right)^2+
            \left(d_2(x,y)\right)^2
        }
    \end{align}
    is also a pseudo-metric.
\end{lemma}

\begin{proof}
    The proof is straightforward by using the Cauchy-Schwarz inequality.
\end{proof}

Finally, we combine $\drthree$ and $\dsone$ into a pseudo-metric on our configuration space $\cspace = \SE(3) \subset\bbR^6$. Define $\distc : \cspace \times \cspace \to \bbR$ as
\begin{align}
    \distc(x,y) =
    \sqrt{
        \left(\drthree(x,y)\right)^2
    + \sum_{i=4}^6 \left(\dsonei(x,y)\right)^2}
\end{align}
\begin{proof}
    Follows immediately by applying Lemma~\ref{lem:combine-pseudometrics} four times.
\end{proof}
 
We choose this $\distc$  to be our distance function as it is very simple to compute, benefits the manifold search method described in Section~\ref{sec:CMG-RRT}, and it is a sufficient choice to show that our proposed method is probabilistically complete.

Using the distance function $\distc$, we can define an open ball of radius $r$ 
centered at a point~$q$,
$\B_r(q)$, as the set of all points whose distance $\distc$ from $q$ is less than~$r$.
Note that such balls are not necessarily path-connected in the natural Euclidean topology. See Figure~\ref{fig:ball} for an example. As a consequence, in the topology induced by $\distc$, balls $\B_r(q)$ are not necessarily convex.
\vspace{-10pt}
\begin{figure}[!t]
    \centering
    \subfloat{\includegraphics[width=0.4\linewidth]{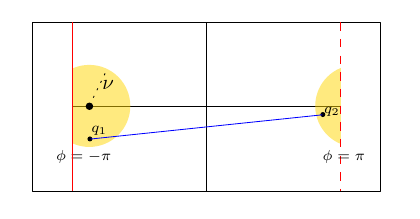}} \hspace{2pt} \subfloat {\includegraphics[width=0.4\linewidth]{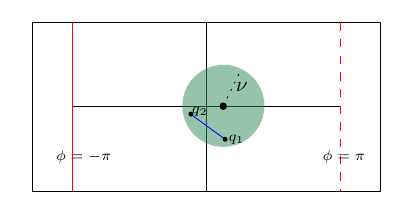}
        
    }\caption{Illustrations of projections of $\B_{\nu}(\cdot)$ balls onto $ [-\pi,\pi)\times \mathbb{R}$. The 
    center of the yellow ball is located in a subset of  $[-{\pi},\pi)\times \mathbb{R}$ such that the ball is not convex. Therefore, there exist points $q_1,q_2$ in the ball such that the straight line segment (in blue) connecting them does not lie in the ball.  
    With high probability (specified in Lemma~\ref{lemma:prob_eucl}), a random rotation will shift the center of the ball to a new position (see the green ball) such that the ball is convex.}
\label{fig:ball}
\end{figure}

\subsection{Signed Distance Functions (SDF)}
\label{ssec:prel-sdf}

Signed distance functions (sometimes also referred to as signed distance fields, or SDFs) are a common approach for representing geometry in computer graphics~\cite{jones20063d}. The geometry is represented as the zero-level set of its SDF.
Formally, assume that we have some compact manifold\footnote{We assume that the manifold is compact and without a boundary, hence, water-tight~\cite{sulzer2024survey}.} $M\subseteq \bbR^3$. In this work, we assume that this manifold $M$ is given as a triangle soup. The signed distance function (SDF) of $M$ is a function $F_M:\bbR^3 \to \bbR$ that maps a point $p\in\bbR^3$ to the distance to its closest point
on the manifold $M$. The distance is signed positive if $p$ is outside $M$, zero if $p\in M$ and negative if $p$ is inside $M$.

A straightforward approach for computing the SDF is to directly query the underlying geometry for the distance. For example, one may compare the query point against all triangles in the triangle soup representation of $M$ to find the nearest distance, and use the parity of the number of intersections of an arbitrary ray with $M$ to determine the sign of the SDF for that query point~\cite{chang2008computing,krayer2019generating}. While there are data structures and algorithms that significantly speed up this query~\cite{dong2018psdf,krayer2019generating,sigg2003signed}, a common approach is to simplify this SDF by evaluating it on a grid, and use that grid as an approximation for the SDF value when querying new points~\cite{tian2022assemble_them_all,zhao2005fast}.
In recent years, there has also been an emergence of deep-learning approaches for approximating the SDF of a given object~\cite{davies2020effectiveness,driess2022learning,jacquet2025neural,li2024representing,park2019deepsdf,shim2023diffusion,yang2025contactsdf}.

One straightforward application of the SDF is as a tool for collision detection. Collision detection is the task of determining whether two (or more) objects overlap~\cite{jimenez20013d}. We can 
represent one object as a collection of points, and evaluate each point in the SDF of the second body~\cite{bender2014continuous,bertiche2021neural,jacquet2025neural,koschier2016hierarchical,macklin2020local,xu20166}. If there are any points with negative distance, then they penetrate the second object and thus a collision occurs. Conversely, if all points have non-negative signed distance, we assume that no collision occurs if the point sample set is sufficiently dense.

\subsection{Problem Statement}

Finally, we formally define the problem we address in this work. 
Assume that $M_1, M_2\subseteq \bbR^3$ are two rigid bodies, which are compact sub-manifolds, with or without a boundary. In this work, we assume that both bodies are scaled down such that they fit in the unit sphere in $\bbR^3$. We assume we have some digital representation of those $M_1$ and $M_2$ for which we can sample random points on the bodies' boundary, 
and for which we can efficiently and effectively evaluate the SDF $F_{M_1}:\bbR^3 \to \bbR$ that is defined in Section~\ref{ssec:prel-sdf}. We assume that $M_1$ is static, and is referred to as the \emph{obstacle(s)}, and $M_2$ can freely translate and rotate in the workspace, and is referred to as the \emph{robot}.
We also define the following function $F_{M_2\to M_1}:\SE(3)\to\bbR$, which is the smallest signed distance of a point on the manifold $M_2$ transformed by a configuration $q\in \SE(3)$ from the static obstacle $M_1$. Formally:

\begin{align}
    F_{M_2\to M_1} (q) = \min_{p\in M_2} F_{M_1} (q\cdot p)\;,
\end{align}
where $q\cdot p$ is the application of a rigid body transformation $q\in \SE(3)$ on a three-dimensional point $p\in \bbR^3$.

The \emph{$\delta$-free space}, denoted by $\free \subseteq \cspace$ is defined as
\begin{align}
    \free \coloneqq \{q \in \cspace \ \vert\ F_{M_2\to M_1}(q) > -\delta \}\; ,
\end{align}
i.e, the set of all configurations for which the robot's penetration into obstacles is at most $\delta$.
We refer to this $\delta$ as the \emph{allowance}, or the allowed penetration.

A \emph{motion-planning problem} is implicitly defined by the triplet $(\free, \qstart, \qgoal)$, with $\qstart, \qgoal \in \free$. A solution to such a problem is a continuous path that moves the robot from the initial configuration $\qstart$ to the goal $\qgoal$ while avoiding collision with obstacles. Formally, a valid path is a continuous\footnote{Note that $\gamma$ is not necessarily continuous under the natural Euclidean topology, but rather in the topology induced by the pseudo-metric $\distc$.} map $\gamma:[0,1]\to\free$, such that $\gamma(0)=\qstart$ and $\gamma(1) = \qgoal$.

\section{Critical-Manifold Guided RRT (\cmg)}
\label{sec:CMG-RRT}

In this section, we introduce our proposed algorithm.
We begin with the motivation in Subsection~\ref{sub:motivation}, identifying the challenges inherent in tight assembly planning and motivating the need for a more principled solution.
We then summarize the \trrrt algorithm~\cite{livnat2024tight}, which addresses these challenges using contact-aware exploration.
Finally, in Subsection~\ref{sub:cmg-rrt}, we present our new method, \emph{Critical-Manifold Guided RRT} (\cmg), and describe its core principles and implementation details.

\subsection{Motivation}
\label{sub:motivation}

The introduction motivates tight assembly planning at a high level. We revisit the motivation here to make explicit the algorithmic bottleneck that \cmg is designed to resolve: in mixed wide-tight settings, the difficulty is not only in generating contact-rich motions once near obstacles, but in \emph{reaching} the relevant near-contact regions efficiently.

Sampling-based planners such as \rrt~\cite{lavalle1998rapidly} and \prm~\cite{kavraki1996probabilistic_prm} perform well when feasible paths have wide clearance. In tight phases, simulation-based~\cite{tian2022assemble_them_all} and contact-aware methods~\cite{lefebvre2005active_compliant_motion} can exploit local geometric constraints. Still, many practical assembly tasks interleave wide-clearance motion with short, critical near-contact transitions whose locations are unknown a priori. Uniform sampling in $\SE(3)$ is therefore inefficient, as most samples fall in regions that do not contribute to progress.

Several methods bias exploration toward narrow passages, e.g., by detecting geometric cues for tunnels~\cite{zhang2020cspace_tunnel_discovery}. While effective on certain families, such cues may be weak or absent in puzzles like the Elk (as mentioned in~\cite{zhang2020cspace_tunnel_discovery}), motivating a bias that depends only on a general distance-to-contact oracle rather than problem-specific features.

Our objective is thus to leverage the structure of the \emph{critical manifold} to concentrate sampling in the small subset of $\SE(3)$ that enables tight transitions, while still allowing efficient exploration of wide regions. This motivates the subdivision-based sampling scheme introduced below.

\subsection{Critical-Manifold Guided RRT (\cmg)}
\label{sub:cmg-rrt}

As a brief reminder, TouchRoll-RRT (\trrrt)~\cite{livnat2024tight} augments a standard \rrt in $\SE(3)$ with a contact-aware extension procedure based on signed distance function (SDF) queries. During tree expansion, \trrrt steers toward a random sample as usual when the motion has clearance; however, when an extension reaches the $\delta$-vicinity of obstacles, it identifies contact points via SDF values and uses the corresponding SDF gradients to maintain sliding/rolling motion along a local contact critical manifold. In effect, this reduces the number of degrees of freedom during tight phases and enables traversal through narrow C-space tunnels that defeat a purely free-space \rrt.

Despite this manifold-aware \emph{extension} step, \trrrt still \emph{samples} uniformly in $\SE(3)$, and therefore benefits from its contact-aware machinery only when random samples happen to induce extensions near the critical manifold. As a result, many samples do not contribute to progress and may repeatedly grow the tree into misleading dead ends. We address this mismatch by modifying the sampling process itself: \cmg adaptively concentrates sampling near the critical manifold using a hierarchical subdivision scheme.

\begin{figure}[h]
    \centering
    \includegraphics[width=1.0\linewidth]{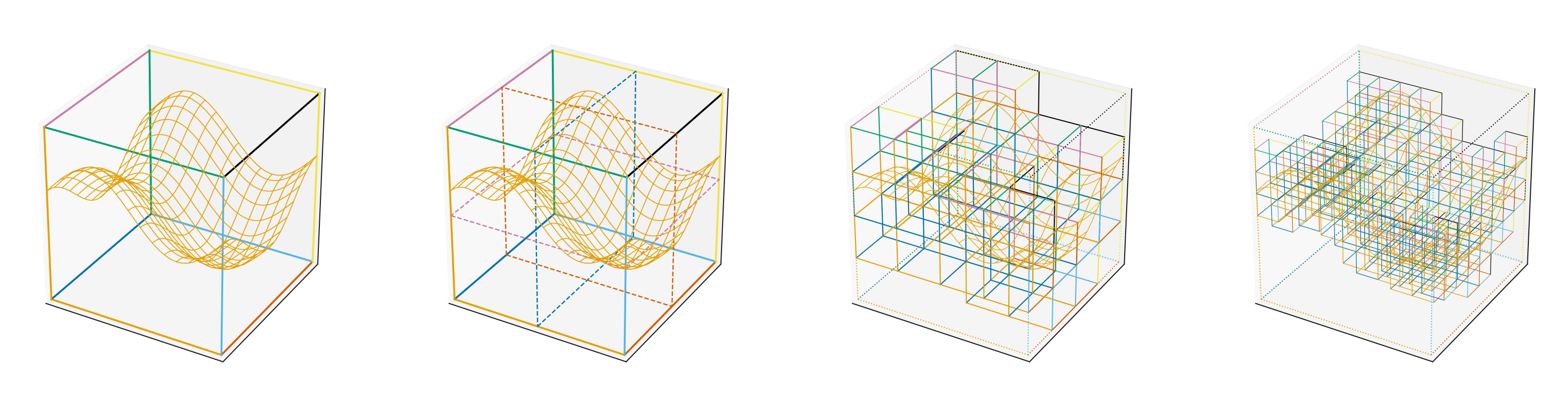}
    \caption{Illustration of the \textsc{Refine} (subdivision) procedure in $3D$. In \cmg, this refinement takes place in $6D$.}
    \label{fig:subdivision}
\end{figure}

Figure~\ref{fig:subdivision} illustrates the central mechanism in \cmg: an adaptive subdivision that progressively filters the configuration space to retain only regions that may lie near the critical manifold.\footnote{We note that the number of boxes maintained depends on the Hausdorff measure of the critical manifold and is $O(1/\delta^k)$ with $k\leq5$~\cite{bilevich2025note}, which is asymptotically better than naively splitting the configuration space into $O(1/\delta^6)$ boxes.}
We next formalize the resulting planner and its refinement search schedule.

\begin{algorithm}[ht]
  \caption{Critical-Manifold Guided RRT}
  \label{alg:cmg-rrt}
  \begin{algorithmic}[1]
    \Require $\eta > 0, \delta > 0$
    \Require $\qstart, \qgoal \in \free$
    \State $V \gets \{\qstart\}$, $E \gets \emptyset$
    \State $\Texpand\gets0, \Trefine\gets0, \factorref\gets 10$
    \State $T \gets 0$
    \While{$T < \LIMIT$}
      \If{$\Trefine \cdot \factorref < \Texpand$}
        \State $\frefine()$
        \State $\fupdate(\Trefine)$
      \EndIf
      \If{$i \bmod \Nrot = 0$}
        \State $\frandomrotate()$
      \EndIf
      \State $s \gets \fsample()$
      \State $\qnear \gets \arg\min_{q \in V} (\distc(q, s))$
      \State $v \gets \fdirection(q_{\mathrm{near}}, s)$
      \State $\qnew \gets \fextend(\qnear, v, \eta)$
      \If{\textsc{IsEdgeValid}($\qnew, \qgoal$)}
        \State \Return \textsc{GetPath}$(V, E, q_{\mathrm{start}}, q_{\mathrm{goal}})$
      \EndIf
      \State $V \gets V \cup \{\qnew\}$
      \State $E \gets E \cup \{(\qnear, \qnew)\}$
      \State \textsc{Update}($\Tsearch$)
      \State $T \gets \Tsearch + \Trefine$
    \EndWhile
    \State \Return \textsc{None}
  \end{algorithmic}
\end{algorithm}

At a high level, \cmg replaces uniform sampling by a subdivision-based sampler~\cite{bilevich2025note} that concentrates samples near the critical manifold.
The idea is to recursively subdivide $\SE(3)$ into shrinking axis-aligned boxes, discard boxes whose configurations are guaranteed not to intersect the critical manifold, and sample only from the remaining boxes (see Figure~\ref{fig:subdivision}). 
A more formal description follows.

\smallskip 

\noindent
\textbf{SDF and distance computation }
In the original \trrrt~\cite{livnat2024tight}, the SDF $F_{M_1}$ was approximated by a neural network, evaluated via GPU forward/backward passes~\cite{davies2020effectiveness,park2019deepsdf}.
In \cmg we instead use a three-dimensional grid with trilinear interpolation~\cite{frisken2000adaptively};
gradients are computed by centered finite differences.
This yields faster query times and allows for efficient parallelization on CPU cores.

\smallskip

We now provide a high-level overview of the algorithmic flow, 
highlighting how refinement, sampling, and tree expansion interact, 
and referring to Algorithms~1-3 for precise pseudo-code.

\smallskip 

\noindent
\textbf{Subdivision module}
The configuration space is six-dimensional with ranges
$x,y,z\in[-1,1]$, $\phi,\psi\in[-\pi,\pi)$, and $\theta\in[-\pi/2,\pi/2]$. The 
\textsc{Subdivision} module
maintains a finite set of axis-aligned boxes $\mathbb{B}$; each box $b\in\mathbb{B}$ is represented by its center and shares a common diagonal length~$D$.
During each call to \textsc{Refine} (Algorithm~\ref{alg:refine}), the boxes are bisected along their longest dimension, and each child box is tested for potential intersection with the contact manifold.
Boxes for which $|\robotsdf(\mathrm{Center}(b))| < D/2$ are retained.
\textsc{Sample}
(Algorithm~\ref{alg:sample}) selects a box uniformly from $\mathbb{B}$ and then selects a configuration uniformly from within that box. 

\begin{algorithm}[t]
\caption{\textsc{Refine}(\textsc{Subdivision})}
\label{alg:refine}
\begin{algorithmic}[1]
\Require Box set $\mathbb{B}$; distance oracle ${\sf dist}(\cdot)$
\State $a \gets \arg\max_{\text{axis}\in\{x,y,z,R_x,R_y,R_z\}}$(side length of boxes in $\mathbb{B}$) \Comment{largest dimension}
\State $\mathbb{B}_{\text{new}} \gets \emptyset$
\For{each $b \in \mathbb{B}$}
  \State Split $b$ into two children $b^{-},b^{+}$ by bisecting along axis $a$
  \For{each $b' \in \{b^{-},b^{+}\}$}
     \State $q_c \gets center(b')$; \quad $D \gets Diagonal(b')$
     \State $d \gets F_{M_2\to M_1} (q_c)$
     \If{$d < D/2$}
        \State $\mathbb{B}_{\text{new}} \gets \mathbb{B}_{\text{new}} \cup \{b'\}$ \Comment{keep potentially near contact}
     \EndIf
  \EndFor
\EndFor
\State $\mathbb{B} \gets \mathbb{B}_{\text{new}}$
\end{algorithmic}
\end{algorithm}

\begin{algorithm}[t]
\caption{\textsc{Sample}(\textsc{Subdivision})}
\label{alg:sample}
\begin{algorithmic}[1]
\Require Nonempty box set $\mathbb{B}$
\State $b \gets$ \textbf{uniformly at random} select a box from $\mathbb{B}$
\State $q \gets$ \textbf{uniformly at random} sample a configuration from $b$ 
       \Comment{independent uniform over each coordinate range of $b$}
\State \Return $q$
\end{algorithmic}
\end{algorithm}

\smallskip 

\noindent
\textbf{Algorithm description}
\cmg begins with several iterations of \textsc{Refine}.
Then, during search, whenever the relative computational budget spent on refinement falls below a threshold, another refinement step is invoked.
This balances refinement cost with search cost: easy problems require few refinements, while hard problems trigger more refinement steps, effectively shrinking boxes toward the contact manifold and yielding highly informative samples.

This refinement-driven sampling accelerates performance on previously solved problems (e.g., the Alpha puzzle) and, crucially, enables solving significantly harder instances such as the Elk puzzle (Fig.~\ref{fig:elk-time-series}).
Zhang~et~al.~\cite{zhang2020cspace_tunnel_discovery} identified Elk as a particularly challenging case due to the absence of identifiable per-piece geometric features and the presence of multiple tunnels involving disjoint contact points from both pieces simultaneously.
Our manifold-guided sampling circumvents these limitations by exploring the C-space directly along its critical manifold, rather than relying on local structural cues.

\section{Probabilistic Completeness}
\label{sec:completeness}

We derive a probabilistic completeness proof for \cmg. 
    We assume that there exists a valid path $\gamma: [0,1]\to \free$, where $\free$ is the set of all configurations for which the robot’s penetration into the obstacles
is at most $\delta>0$. 

We follow our general approach as described in~\cite{kleinbort2018probabilistic,solovey2020revisiting}, but we adapt it to handle rotations explicitly, and we connect it to the boxes that arise in the subdivision process.
We define a sequence of balls covering the path and show that with high probability, \cmg will generate a path that goes through the union of these balls in the order of the sequence. 
We show that this probability converges to one as the number of samples tends to infinity.

Denote by~$L$ the length of the path $\gamma$ under the pseudo-metric $\distc$. 
Let~$m=\frac{5L}{\nu}$, where $\nu = \min(\delta, \eta)$,  and $\eta$ is the {maximal} step size used by the algorithm. Then, define a sequence of ${m}+1$ points $q_0=\gamma(0),\ldots,q_{m}=\gamma(1)$ along $\gamma$, such that the~$\distc$ length of the sub-path between every two consecutive points is $\nu/5$. Therefore, $\distc(q_i,q_{i+1})\leq\nu/5$ for every $0\leq i< m$. Next, we define a set of ${m}+1$ balls of radius $\nu/5$, centered at these points.
We now prove that with high probability, \cmg will generate a path that goes through these balls.

First, we claim that for every configuration $q\in \cspace$, we can sample a random 
rotation $R$ such that with high probability the ball $\B_{\nu}(q')$ centered at $q'= R\cdot q$  is convex.
Denote this probability by $p_{\text{rotate}}$.

\begin{lemma}
    Fix some $q\in\cspace$. Let $R\in\SO(3)$ which is sampled uniformly, i.e., $R\sim \calU(\SO(3))$. Then, with probability at least $1 - \frac{2}{\pi}\nu$, the ball $\B_{\nu}(R\cdot q)$ is convex.
    \label{lemma:prob_eucl}
\end{lemma}

 \begin{proof}
   Throughout the proof, with a slight abuse of notation, $R$ denotes both a rotation and the corresponding rigid-body transformation, by embedding $R$ in $\SE(3)$ with zero translation.
    Let $\pi_{\SO(3)}:\SE(3)\to\SO(3)$ be the projection from $\SE(3)$ to its orientation subgroup $\SO(3)$.
    We first notice that sampling $R$ uniformly from $
    \SO(3)$ is equally likely to sampling $R'=R\cdot \pi_{\SO(3)}(q^{-1})$.
    We also notice that the ball's $\B_\nu$ convexity does not depend on the translation (which is Euclidean) but only on the orientation.
    Hence, it suffices to
    prove that the ball $\B_\nu(R)$ is convex.

    Let $\pi_{\phi}, \pi_{\psi}:\SO(3) \to [-\pi, \pi)$ be the projection of the orientation into the roll and yaw, respectively. Note that when both angles have $\pi_\phi(R),\pi_\psi(R) \in (-\pi + \nu, \pi -\nu)$, by definition of $\distc$, for any two points $x,y \in \B_\nu(R)$, the pseudo-metric $\distc$ coincides with the Euclidean 
    metric on $\bbR^6$. Thus, the ball $\B_\nu(R)$ coincides with the Euclidean ball, and as such, it is convex.

    Using the volume element defined in Section~\ref{ssec:so3}, we can bound the probability for either 
    $\pi_\phi(R),\pi_\psi(R) \not\in (-\pi + \nu, \pi -\nu)$:

    \begin{align}
        \bbP[&\mathrm{either}\ \pi_\phi(R),\pi_\psi(R) \not\in (-\pi + \nu, \pi -\nu)
        ] \nonumber \\
        &\leq \bbP\left[
            \pi_\phi(R) \not\in (-\pi + \nu, \pi -\nu)
        \right] + 
        \bbP\left[
            \pi_\psi(R) \not\in (-\pi + \nu, \pi -\nu)
        \right] \nonumber \\
        &= \bbP[
            \pi_\phi(R) \in [-\pi, -\pi +\nu]
        ] + 
        \bbP\left[
            \pi_\phi(R) \in [\pi-\nu, \pi)
        \right] \nonumber \\
         &\ \ \ \ \ + \bbP[
            \pi_\psi(R) \in [-\pi, -\pi +\nu]
        ] + 
        \bbP\left[
            \pi_\psi(R) \in [\pi-\nu, \pi)
        \right] \nonumber \\
        &= 4\cdot\int_{-\pi}^{-\pi+\nu}\int_{-\pi/2}^{\pi/2}\int_{-\pi}^\pi \frac{1}{8\pi^2}\cos\theta \ d\phi\ d\theta\ d\psi = 4\cdot \frac{1}{8\pi^2}\cdot 2 \cdot 2\pi \cdot \nu = \frac{2}{\pi}\nu \nonumber\;.
    \end{align}
    Notice that due to symmetry, the four probability computations amount to the same triple integral.
    Hence, the probability that the ball $\B_\nu(R)$ is Euclidean (and thus convex) is at least $1 - \frac{2}{\pi}\nu$. 
 \end{proof}

Next, we prove that if \cmg has reached a certain ball, with probability greater than zero, it will reach the next consecutive ball.

\begin{lemma}
Suppose that \cmg  has reached $\B_{\nu/5}(q_i)$, that is, $T$ contains a vertex $q'_i$ such that 	$q'_{i}\in \B_{\nu/5}(q_i)$. 
With high probability \cmg will reach $\B_{\nu/5}(q_{i+1})$.

\label{lem:almost_geometric_step}
\end{lemma}

\begin{proof}

\begin{figure}[h]
    \centering
    \includegraphics[width=0.4\linewidth]{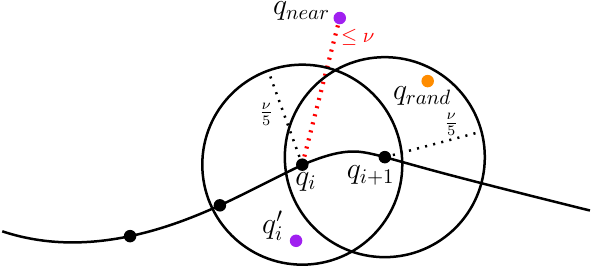}
    \caption{Illustration of Lemma~\ref{lem:almost_geometric_step}. Note that the balls are not necessarily convex. However, with probability $p_{\text{rotate}}$, we will choose a random rotation such that, in each transition step, the relevant balls are convex.}
    \label{fig:lemma_step}
\end{figure}
Suppose that $q_{\text{rand}}$ is drawn such that $q_{\text{rand}}\in \B_{\nu/5}(q_{i+1})$.
Denote by $q_{\text{near}}$ the nearest neighbor of $q_{\text{rand}}$ among the \cmg vertices.
See Fig.~\ref{fig:lemma_step} for an illustration.
Then, from the definition of $q_{\text{near}}$, it follows that $\distc( q_{\text{near}}, q_{\text{rand}}) \leq \distc(  q'_i , q_{\text{rand}})$.

From 
$\distc( q_{\text{near}}, q_{\text{rand}}) \leq \distc(  q'_i , q_{\text{rand}})$
and the triangle inequality, we have:
\begin{align*}
	\distc(q_{\text{near}}, q_i) &\leq 
	\distc(q_{\text{near}}, q_{\text{rand}})+ \distc(  q_{\text{rand}}, q_i)
	\\
	&\leq 
	\distc (q'_i,q_{\text{rand}}) + \distc(q_{\text{rand}}, q_i) .
	\end{align*}
	From the triangle inequality, we have that 
	\[\distc( q_{\text{rand}} ,q_i)\leq \distc(  q_{\text{rand}}, q_{i+1} ) + \distc(  q_{i+1}, q_{i}),\]
    \[\distc (  q_{\text{rand}},q'_i) \leq \distc( q_{\text{rand}} ,q_i) + \distc(  q_i,q'_{i})).
	\]
	Therefore, and since $\distc(\cdot, \cdot)$ is symmetric:
	\begin{align*}
	\distc(  q_{\text{near}} ,q_i ) &\leq 
	\distc (q'_i ,q_i ) + 2\cdot(\distc( q_{i+1} ,q_{\text{rand}}) + 
	\distc(q_{i+1},q_{i})) \leq 5\frac{\nu}{5} = \nu.
	\end{align*}

    Note that by applying a random rotation $R$, we split and reform the ball $\B_\nu(q_i)$. 
    From Lemma~\ref{lemma:prob_eucl}, the probability that the random rotation chosen will cause $\B_\nu(q_i)$ to be convex is $p_{\text{rotate}}$. 

    If $\B_\nu(q_i)$ is convex then since 
     $q_{\text{near}}\in\B_\nu(q_{i})\subseteq\free$ and since $q_{\text{rand}}\in\B_\nu(q_{i})\subseteq\free$ we obtain $\overline{q_{\text{near}}q_{\text{rand}}}\subseteq \B_\nu(q_i)\subseteq\free$. 
    
    Also, the distance between
    $q_{\text{near}}$ and $q_{\text{rand}}$ is at most $\eta$ since: 
	$\distc( q_{\text{rand}} ,q_{\text{near}}) \leq \distc( q_{\text{rand}} , q'_i )
\leq \distc( q_i',q_i) +  \distc (q_i ,q_{i+1} )  + \distc ( q_{i+1}, q_{\text{rand}}) \leq 3\cdot\frac{\nu}{5} < \nu \leq \eta.$ 
The fact that $\distc( q_{\text{near}} ,q_{\text{rand}}) \leq \eta$, means that $q_{\text{new}} = q_{\text{rand}}$.

Finally, we bound the probability to sample $q_{\text{rand}}$ 
such that it lies in $\B_{\nu/5}(q_{i+1})$.
After a finite number of refinement steps, the diagonal $D$ of a box will be at most $\nu/5$. When $D\leq \nu/5$, and assuming that $\B_\nu(q_i)$ is convex, then each such ball will contain at least one box $b\in \mathbb{B}$.
Since we sample uniformly from the boxes in the box set $\mathbb{B}$, the probability 
$p_{\text{sample}}$
to sample a point within $\B_{\nu/5}(q_{i+1})$ given that $\B_\nu(q_i)$ is convex is at least $ |b|/|\mathbb{B}|$, where $|b|$ is the Lebesgue measure\footnote{Note that we use the Lebesgue measure here since we sample points in $\bbR^6$.} of a box $b\in \mathbb{B}$ and $|\mathbb{B}|$ is the Lebesgue measure of the box set $\mathbb{B}$.

Thus, the probability $p$ that the straight line segment between  a random sample $q_{\text{rand}}$ and its nearest neighbor $q_{\text{near}}$ in $T$ lies entirely in $\free$ is
$p =p_{\text{rotate}}\cdot p_{sample} >0$.

\end{proof}
We now prove our main theorem.
\begin{theorem}
	The probability that \cmg fails to reach $q_{\text{goal}}$ from $q_{\text{init}}$ after $k$ iterations
	is at most $ae^{-bk}$, for some constants $a,b \in \bbR_{>0}$.
	\label{thm_main}
\end{theorem}

\begin{proof}

Assume that $\B_{\nu/5}(q_{i})$ already contains a \cmg vertex.  From Lemma~\ref{lem:almost_geometric_step}, with probability $p>0$ in the next iteration a \cmg vertex will be added to $\B_{\nu/5}(q_{i+1})$. 

Reaching $q_{\text{goal}}$ from $q_{\text{init}}$ with \cmg requires
repeating this step $m$ times, transitioning from $q_i$ to $q_{i+1}$ for $0\leq i < m$. 
With probability $p_{\text{rotate}} \geq p$ 
the $m$th~ball $\B_{\nu/5}(q_{\text{goal}})$ is convex, and then any line segment from a point in the ball to its center $\qgoal$ is in $\free$.
Therefore, this process can be described as $k$
Bernoulli trials with success probability $p$, where we would like to bound the probability to obtain $m+1$ successful outcomes ($m$ for reaching $\B_{\nu/5}(\qgoal)$ from $\B_{\nu/5}(\qstart)$ and another successful outcome in choosing a random rotation such that $\B_{\nu/5}(\qgoal)$ is convex and thus any straight line of points in it is in $\free$, as described before). 
By defining success to be $m+1$ successful outcomes, we obtain an upper bound on the probability of failure, as 
the process may end after less than $m+1$ successful outcomes. 
	
As in the analysis of~\cite{kleinbort2018probabilistic}, we can bound the  probability of failure. That is, the probability that the process does not reach state $(m+1)$,
	after $k$ steps. Let $X_k$ denote the number of successes in $k$ trials, then

    \begin{align}
		\Pr[&X_k < (m+1)] = \sum_{i=0}^{m}{\binom{k}{i}p^i(1-p)^{k-i}}\nonumber\\
		&\leq \sum_{i=0}^{m}{\binom{k}{m}p^i(1-p)^{k-i}}\nonumber \leq \binom{k}{m}\sum_{i=0}^{m}{(1-p)^{k}} \nonumber \\
		&\leq \binom{k}{m}\sum_{i=0}^{m}{(e^{-p})^{k}}  
		= \binom{k}{m}(m+1){e^{-pk}} \nonumber \\
		&= \frac{\prod_{i=k-m-1}^k{i}}{m!}{(m+1)e^{-pk}}  
		\leq \frac{1}{m!}k^{m+1}(m+1){e^{-pk}}\nonumber,  
	\end{align}
    
	{where the transitions rely on (i) $m\ll k$, (ii)  $p<\frac{1}{2}$,
	and (iii)  $(1-p)\leq e^{-p}$.}

	As $p,m$ are fixed and independent of $k$, the expression $\frac{1}{m!}k^{m+1} (m+1){e^{-pk}}$ decays to zero exponentially with $k$.	Therefore, \cmg is probabilistically complete.
\end{proof}

\section{Experiments and Results}
\label{sec:results}
In this section, we present comparisons of \cmg with state-of-the-art benchmarks and results. Specifically, we refer to a subset of $16$ instances from the benchmark presented in~\cite{tian2022assemble_them_all} which are challenging because they require a non-trivial combination of translation and rotation. In addition, we include the Elk puzzle, mentioned in~\cite{livnat2024tight,zhang2020cspace_tunnel_discovery} which is notably difficult.

\subsection{Implementation Details}
Our open-source software is written in Python and is available online (see project page, footnote 2). The code and all evaluations were run on a Linux machine with an Intel Core i7-12700 CPU. 
The tree exploration is parallelized among 16 cores.
All models are scaled beforehand such that they fit inside the three-dimensional unit sphere $\bbS^2\subset\bbR^3$.

\smallskip 

\noindent
\textbf{Collision Detection }
To perform collision detection and to find contact points, we do the following. At the beginning of the algorithm, we sample $10,000$ points on the boundary of the robot $M_2$. Then, for a configuration $q\in \SE(3)$, we transform those sampled points by $q$ and evaluate the SDF of $M_1$. That SDF is implemented as a $240^3$ grid. The SDF construction takes  $7.4$ seconds on average and at most $12.9$ seconds for all tested models. We evaluate grid points in parallel using OpenMP~\cite{dagum1998openmp}, and perform distance queries with CGAL's Axis Aligned Bounding Box (AABB) tree~\cite{cgal:atw-aabb-26a}. Points that have an absolute signed distance less than a threshold $\delta = 0.005$ are considered contact points. Points with signed distance less than $-\delta$ are considered in penetration, and we report a collision. 

Identifying too many contact points can artificially eliminate all the robot's degrees of freedom, particularly when nearby points induce nearly identical gradients. Therefore, after identifying all potential contact points, we cluster them by proximity of their gradients using the K-means algorithm~\cite{ahmed2020k,macqueen1967some} 
and achieve representative contact points for calculating of the tangent or retract direction. 
 
\subsection{Evaluation Against Baseline Methods}

We compare our results against four representative sampling-based methods: PRM$^*$~\cite{karaman2011sampling}, BIT$^*$~\cite{gammell2020batch}, BKPIECE~\cite{csucan2009kinodynamic}, and RRTConnect~\cite{kuffner2000rrt}, which are all implemented in the OMPL library~\cite{sucan2012the-open-motion-planning-library}, as well as four methods dedicated to tight assemblies and puzzles: Tian et al~\cite{tian2022assemble_them_all}, BK-RRT~\cite{zickler2009efficient}, Zhang et al.~\cite{zhang2020cspace_tunnel_discovery}, and the original TR-RRT~\cite{livnat2024tight}.
We evaluated our method on the dataset provided by Tian~et~al.~\cite{tian2022assemble_them_all} (see Figure~\ref{fig:dataset}). 
Specifically, we deal with their \emph{rotational assemblies}, which are problems that require a combination of simultaneous translation and rotation to disassemble the parts. 
This dataset is comprised of overall $24$ instances, divided into three categories (of eight instances each): \texttt{screws}, \texttt{puzzles}, and \texttt{others}. Since our method deals with cases where there is only a discrete number of contact points at all times, we have not evaluated it on the screws category, which have contact surfaces (i.e., the set of all contact points in the workspace is a $2$-manifold). 
Finally, we have tested on the Elk puzzle, which was first presented as an unsolved benchmark in~\cite{zhang2020cspace_tunnel_discovery}, and later partially solved in~\cite{livnat2024tight}, but which, to the best of our knowledge, has never been completely solved automatically. Our method, \cmg, solves this puzzle with $100\%$ success rate and with an average time of $156$ minutes.

\begin{table}[!t]
\caption{Results on 16 rotational assemblies from~\cite{tian2022assemble_them_all} under \emph{puzzle} and \emph{other}.}
\label{table:success_rate}
\centering
\renewcommand{\arraystretch}{1.12}
\setlength{\tabcolsep}{5pt}

\begin{threeparttable}
\begin{tabular*}{\columnwidth}{@{\extracolsep{\fill}}
    l
    S[table-format=3.1]
    S[table-format=2.1]
@{}}
\toprule
Algorithm & {Success (\%) $\uparrow$} & {AST (min) $\downarrow$} \\
\midrule
PRM*~\cite{karaman2011sampling}            & 0   & {-}\\
BIT*~\cite{gammell2020batch}            & 0   & {-} \\
BKPIECE~\cite{csucan2009kinodynamic}         & 6.9   & 9.6 \\
RRTConnect~\cite{kuffner2000rrt}      & 12.5  & 5.2  \\

\midrule
BK-RRT~\cite{zickler2009efficient}        & 75.0 & 10.9 \\
Tian et al.~\cite{tian2022assemble_them_all}     & 88.0  & 8.7  \\
Zhang et al.\tnote{$\dagger$}~\cite{zhang2020cspace_tunnel_discovery}
                & 100.0 & 37.6  \\
TR-RRT~\cite{livnat2024tight}          & 100.0 & 61.9 \\
                
\midrule
\cmg{}         & \textbf{100.0} & \textbf{3.0}  \\
\bottomrule
\end{tabular*}
\begin{tablenotes}
\footnotesize
\item[$\dagger$] Results reported from the original paper on \emph{puzzles} only, running on a HTCondor cluster. This method, by design, can be applied only for puzzles.
\end{tablenotes}
\end{threeparttable}
\end{table}

We ran each method on each instance $10$ times, with a timeout of $4$ hours.
If, before that timeout, a valid path of motion was found, then we consider this a success.

In Table~\ref{table:success_rate}, we show the success rates and average success time (AST) of the methods. Running time includes the SDF grid construction.
\cmg's success rate outperforms all but the original TR-RRT, consistently solving all evaluated instances, and runs faster than all tested methods, when considering succefull instances.

\section{Discussion}
\label{sec:discussion}
This work targets a fundamental challenge of sampling-based planners in tight assemblies: although feasible motions often concentrate near contact, uniform sampling in $\SE(3)$ spends most effort in regions that do not contribute to progress. \cmg addresses this mismatch by adaptively restricting sampling to a shrinking set of boxes that are provably near contact according to an SDF-based oracle, while preserving probabilistic completeness. Empirically, this bias toward the critical manifold substantially improves performance on challenging rotational assembly benchmarks, solving instances that prior methods have not, including the Elk puzzle.

\begin{figure}[h]    
    \centering     \includegraphics[width=1.0\linewidth]{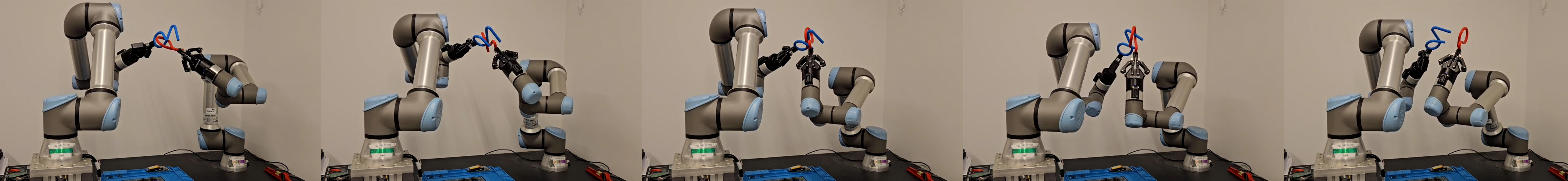}
    \caption{Execution of a \cmg planned trajectory for the \texttt{az} puzzle using two UR5e robotic arms. Arms trajectories are based on a separate continuous IK derived from the exact \cmg trajectory for the free-flying objects.}    
    \label{fig:ur5x2-az}
\end{figure}

Figure~\ref{fig:ur5x2-az} provides a complementary, qualitative validation of the trajectories produced by \cmg. While \cmg plans for free-flying rigid bodies in $\SE(3)$, we executed one planned solution using two UR5e robotic arms by converting the free-flying pose trajectory to continuous joint-space motion using a separate IK-based pipeline~\cite{livnat2025full}. The resulting motion illustrates that the planned path is smooth and physically plausible and that the contact-rich portion of the motion can be tracked without requiring manual modification of the underlying trajectory. We emphasize that this experiment is intended as a qualitative demonstration of executability rather than a full manipulation-planning evaluation.

\cmg has a couple of noticeable limitations. First, our current implementation assumes a grid-based SDF with finite-difference gradients; while this yields predictable query times, the approximation quality depends on grid resolution and scaling. 
Second, our method is designed for rigid-body assemblies with point contacts; extending it to settings with sustained surface contact (e.g., screw-like motions) may require additional progress.

A natural next step is to extend \cmg from two-part to multi-part assemblies, where progress requires simultaneous and mutually consistent contacts among three or more parts. We expect the subdivision-and-prune approach to be particularly valuable in this setting, as coordinated contacts implicitly restrict feasible motion to a small, structured subset of the full configuration space, allowing sampling to focus on regions where progress is possible despite the increased dimensionality.

%
%
\bibliographystyle{IEEEtran}
\bibliography{bibliography}

\end{document}